\documentclass[10pt,twocolumn,letterpaper]{article}

\usepackage{iccv}
\usepackage{times}
\usepackage{graphicx}
\usepackage{comment}
\usepackage{amsmath}
\usepackage{amssymb}
\usepackage{booktabs}
\usepackage{enumitem}
\usepackage{algorithm}
\usepackage[noend]{algpseudocode}
\usepackage{bbm}
\usepackage{bbold}
\usepackage{authblk}

\usepackage[pagebackref=true,breaklinks=true,letterpaper=true,colorlinks,bookmarks=false]{hyperref}
\usepackage{cleveref}
\usepackage{subcaption}

\usepackage{amsthm}
\newtheorem{theorem}{Theorem}

\iccvfinalcopy 

\def\iccvPaperID{538} 

\newcommand{\mypara}{\vspace*{-10pt}\paragraph}

\newcommand{\rand}{\texttt{rand}\xspace}
\newcommand{\traj}{\texttt{traj}\xspace}
\newcommand{\cat}{\texttt{cat}\xspace}
\newcommand{\heur}{\texttt{heur}\xspace}
\newcommand{\datamatch}{\texttt{datamatch}\xspace}
\newcommand{\hum}{\texttt{hum}\xspace}
\newcommand{\yinda}{Zhang et al.~\cite{zhang2016render}\xspace}

\ificcvfinal\fi
\begin{document}

\title{Learning Where to Look: Data-Driven Viewpoint Set Selection for 3D Scenes}
\author{Kyle Genova}
\author{Manolis Savva}
\author{Angel X. Chang}
\author{Thomas Funkhouser}
\affil{Princeton University}
\maketitle

\begin{abstract}

The use of rendered images, whether from completely synthetic datasets
or from 3D reconstructions, is increasingly prevalent in
vision tasks. However, little attention has been given to how
the selection of viewpoints affects the performance of rendered
training sets.  In this paper, we propose a data-driven approach to
view set selection.  Given a set of example images, we extract
statistics describing their contents and generate a set of
views matching the distribution of those statistics.  Motivated by
semantic segmentation tasks, we model the spatial distribution of
each semantic object category within an image view volume.
We provide a search algorithm that generates a sampling of likely
candidate views according to the example distribution, and a set
selection algorithm that chooses a subset of the candidates that jointly cover
the example distribution.  Results of experiments with these algorithms
on SUNCG indicate that they are indeed able to produce view distributions
similar to an example set from NYUDv2 according to the earth mover's distance.
Furthermore, the selected views improve performance on semantic segmentation
compared to alternative view selection algorithms.

\end{abstract}

\section{Introduction}

Rendering of 3D scenes, both synthetic and reconstructed, is a
promising way to generate training data for deep learning methods
in computer vision.  For example, several methods have trained
on rendered images to improve the performance of semantic
segmentation networks \cite{zhang2016render}. The extent to
which the rendered images can be useful as training data
depends on the quality of the match between the rendered data and
real-world images.  For example, when rendering images for training
semantic segmentation algorithms, the statistics of object occurences,
sizes, and placements influence the outcome.

Yet, little previous work has investigated algorithms to select camera
views for the purposes of generating training sets for vision tasks.
Most previous work on view selection has focused on optimizing the
aesthetic properties of rendered images for applications in computer
graphics
\cite{bares2006photographic,bares2000virtual,christie2008camera,gooch2001artistic,olivier1999visual,liu2015composition}
or visible surface coverage for surveillance
\cite{mavrinac2013modeling} and surface reconstruction
\cite{zheng2014patchmatch,schonberger2016pixelwise}.

In this paper, we investigate a new view set selection problem
motivated by generating training sets for computer vision
tasks.  We pose the following view selection problem:
select a set of views whose ``distribution of image content'' best
matches a set of example images.  This problem statement defines
the objective as matching a latent distribution generating the example
set, rather than optimizing a particular function and thus is quite
different than previous work.  This requires choosing a representative
example set, defining a ``distribution of image content,''
and searching the infinite space of view sets for the best match.

In this paper, we investigate one concrete instance of the problem motivated by generating training data for semantic segmentation:
{\em to select a set of views where the pixel distribution of object
observations for specific semantic categories matches that of an
example image set.}
By matching the spatial distributions of object categories in example
images, we may be able to train deep networks more effectively than is
possible now with cameras placed heuristically \cite{zhang2016render}, along
trajectories \cite{mccormac2016scenenet}, or randomly \cite{handa2015scenenet}.


We propose an algorithm with three components to address the problem.
The first reduces a candidate image into a low-dimensional representation that models the distances between the spatial distributions of pixel sets for each semantic category.  The second suggests candidate positions in a 3D computer
graphics scene that are likely to yield rendered images with a given
pixel distribution of semantic categories.  The third uses submodular
optimization to select a set amongst the candidate views to match the
overall latent distribution of the example set.

Results of experiments with
this algorithm demonstrate that it is practical, efficient, and more
effective than traditional approaches at selecting camera views for
the SUNCG dataset \cite{song2017ssc} when asked to match the distribution of
views from NYUDv2 \cite{silberman2012indoor}. We find that our selected views are quantitatively more
similar to NYUDv2's than previous approaches attempting to model the distribution heuristically, and we show that training an
FCN network \cite{long2015fully} on them provides better performance than for
other view sets.  We make the following contributions:
\begin{itemize}
\setlength{\topsep}{0pt}
\setlength{\parsep}{0pt}
\setlength{\parskip}{0pt}
\setlength{\itemsep}{0pt}
\vspace{-0.8em}
\item We are the first to focus on the problem of viewpoint selection for synthetic dataset generation in the context of data-hungry tasks in computer vision.
\item We present a model for viewpoint selection that matches a data prior over object occurence patterns specified through example images of the real world.
\item We introduce an algorithm based on submodular optimization for selecting a set of images approximately optimizing a function measuring deviation from an statistical distribution derived from examples. 
\item  We present results demonstrating that our agent-agnostic algorithms for view set selection outperform previous alternatives that are manually tuned to match a particular agent.
\end{itemize}

\section{Related work}

View selection has been studied in several contexts, including camera placement in graphics, next-best view prediction and sensor placement in robotics, canonical views in perceptual psychology, among many others.

\mypara{Camera optimization in computer graphics.}
Early work in graphics has used information entropy as a measure of view quality to select good views of objects~\cite{vazquez2001viewpoint}.
Another line of work encodes image composition principles to optimize virtual camera placements such that they focus on specified objects in the frame and are judged to be aesthetically pleasing by people~\cite{bares2006photographic,bares2000virtual,christie2008camera,gooch2001artistic,olivier1999visual}.
More recent work has taken a similar approach in the more specific scenario of product images~\cite{liu2015composition}.
Freitag et al.~\cite{freitag2015comparison} compare several viewpoint quality estimation metrics on two virtual scenes, though none of the metrics involve matching real-world data semantics.
A related and rich body of work studies automated camera path planning for real-time graphics systems --- a good survey is given by Jankowski and Hachet~\cite{jankowski2015advances}.  Generally, these systems use manually encoded heuristics to create smooth camera trajectories.  Unlike our work, the focus is not selecting a set of static views, or on matching real-world view statistics based on object semantics.
In general, related work in computer graphics has not considered modeling camera priors based on real-world image data, and it has not evaluated the impact of camera viewpoints on computer vision tasks.

\mypara{Next-best view prediction and camera placement optimization.}
Next-best view prediction has been addressed in the context of various problems in robotics and computer vision: 3D reconstruction~\cite{krainin2011autonomous}, volumetric occupancy prediction~\cite{wu20153d}, and object pose estimation~\cite{doumanoglou2016recovering} among others.
However, the next-best view problem is predominantly addressed in an active sensing context where planning for the next most informative view is desired.  In contrast, we focus on the offline problem of selecting a set of static viewpoints given a 3D indoor scene dataset as input.
Therefore, we are more closely related to work in camera placement optimization.  The camera placement problem is typically cast as a version of the art gallery problem, seeking to minimize the number of cameras needed to ensure visual coverage of a target environment.  There is much prior work in this area --- a recent survey is provided by Mavrinac and Chen~\cite{mavrinac2013modeling}.  Similar problem formulations have been used in robotics for view planning of 3D sensing~\cite{blaer2007data}.  Our problem statement is distinct since we match the distributions of semantically meaningful objects instead of simply optimizing for coverage of a single input environment.

\mypara{Canonical views of objects and scenes.}
Internet image collections of particular objects were shown to mirror canonical viewpoint preferences by people~\cite{mezuman2012learning}.
Other work has conditioned the camera viewpoint on object keypoints to jointly predict the 3D object pose and the viewpoint~\cite{tulsiani2015viewpoints}.
Judgments of preferred views collected through crowdsourcing have been used to train a predictor for preferred viewpoints of objects~\cite{secord2011perceptual}.
Ehinger and Oliva~\cite{ehinger2011canonical} ask people to select a ``desirable'' (i.e., canonical) orientation within 2D panoramas and show that chosen views correlate highly with directions that maximize the visible volume and also coincide with prominent navigatable paths.
Our approach assumes that images taken by people for inclusion in image datasets exhibit a prior over natural viewpoints that we can extract and leverage to select views in 3D scenes.  

\mypara{Rendering synthetic 3D scenes for data generation.}
There is a recent explosion in generation of synthetic training data for many vision tasks including tracking, object recognition, semantic segmentation, pose estimation, and optical flow among others~\cite{choi2015robust,papon2015semantic,handa2012real,handa2014benchmark,handa2015scenenet,mayer2015large,ros2016synthia,shafaei2016play,richter2016playing}.
Work that rendered synthetic 3D images has focused on domains where the camera viewpoint is highly constrained (e.g., driving~\cite{ros2016synthia,shafaei2016play,richter2016playing}), has used manually specified camera trajectories (e.g., by recording first-person user trajectories in the virtual scene, or from camera paths in the real world~\cite{handa2012real,handa2014benchmark}), or has purportedly randomly sampled the space of camera views~\cite{handa2015scenenet}.
No previous paper has focused on view set selection or investigated its impact on trained models.


\begin{figure}
	\centering
	\includegraphics[width=\linewidth]{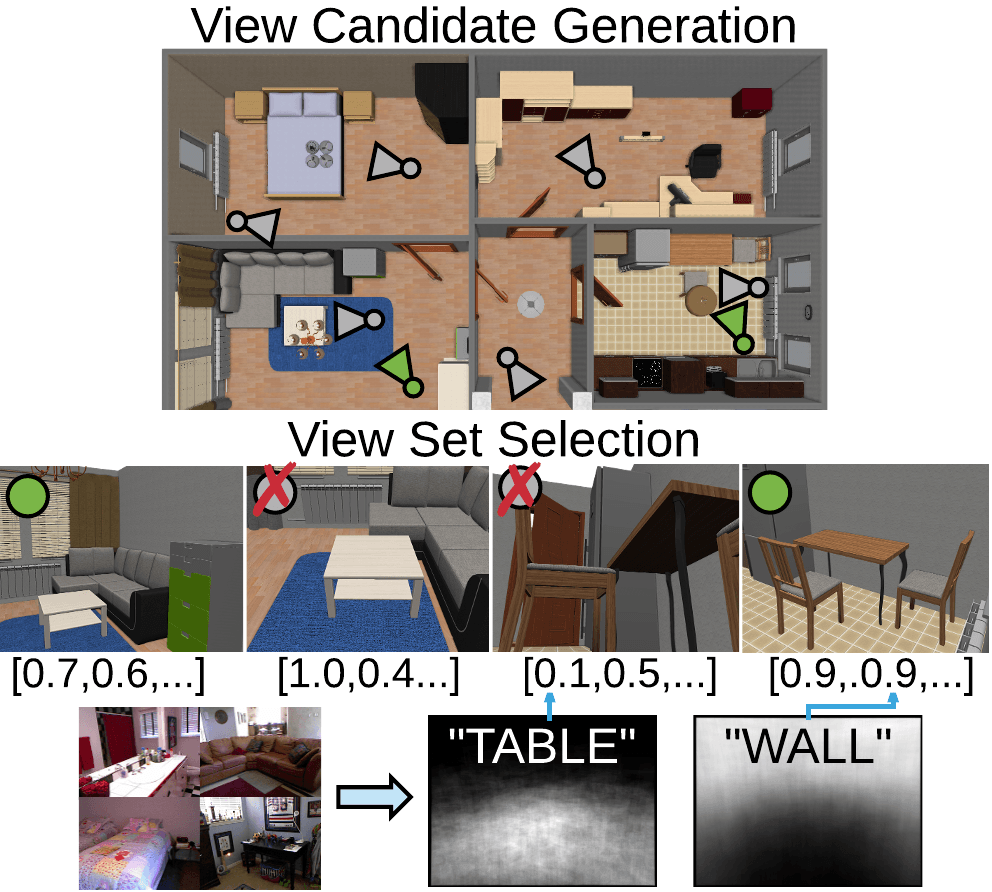}
	\caption{\textbf{Top}: candidate views in a target 3D scene with the selected output view set in green.  \textbf{Bottom:} category pdfs (right) from example set (left) are used to estimate image likelihood along multiple semantic axes.  \textbf{Middle:} a set of views is selected to jointly match the object distribution statistics of the input.}
	\label{fig:problem}
	\vspace{-1em}
\end{figure}

\section{Data-driven viewpoint selection}

In this paper, we investigate a new data-driven approach to view selection in synthetic scenes.  The input to our system is a 3D computer graphics model of an indoor scene (e.g., from SUNCG \cite{song2017ssc}) and a set of example RGB-D images containing semantic segmentations (e.g., from NYUDv2 \cite{silberman2012indoor}).  The output is a set of views (6 DoF camera poses), where the spatial pixel distribution of semantic objects in images rendered from the views ideally matches the distribution observed in the example set (Figure \ref{fig:problem}).

We model the views of the example set as random samples drawn from a latent distribution.   To model that distribution, we represent it as a set of $n$ probability density functions (pdfs) representing the x, y, and depth positions of pixels for each of $n$ semantic classes observed in the example set.  This pixel-level representation captures the spatial layout of different object classes in the example images, as depicted in \Cref{fig:problem} (bottom right shows a 2D representation of the 3D histogram associated with ``table'').

We select a set of views approximately covering an estimate of that distribution with the two step process depicted in Algorithm \ref{RGBDPriorAlg}.  During the first step, we generate candidate camera positions by sampling viewpoints according to pdfs dependent only on depths and categories of visible objects (i.e., the original pdfs with rotations and pixel xy marginalized).  During the second step, we choose a set of candidate views that approximately match the entire pdf as a set using an algorithm based on submodular maximization.  The following three subsections describe the pdf representation and the two algorithmic steps in detail.


\begin{algorithm}
	\caption{Viewpoint Generation with RGBD Priors }\label{RGBDPriorAlg}
	\begin{algorithmic}[1]
		\Function{Generate}{Scene,Room}
		\State Voxelize Room into 1000-10000 voxels $\mathcal{V}$
		\For{$v \in \mathcal{V}$}
		\State $W^{\mathcal{V}}_{v,:} \gets 0$
		\State $\mathcal{K}_: \gets 0$
		\For{$i \in [\# Samples]$}
		\State Sample Point $p^o \sim \mathcal{U}(v)$
		\State Sample Direction $d^o \sim \mathcal{A}$ 
		\State Ray $r \gets (p^o,d^o)$
		\If {$r$.Intersects(Scene)}
		\State$p \gets $ IntersectionPoint()
		\State $c \gets $ ObjectAt($p$).Category()
		\State $d \gets \|p - r_p\|_2$
		\State $W^{\mathcal{V}}_{v,c} \gets W^{\mathcal{V}}_{v,c} + \sum_{x}\sum_{y} f_c(x,y,d)$
		\State $\mathcal{K}_c \gets \mathcal{K}_c + 1$
		\EndIf
		\EndFor
		\State $W^{\mathcal{V}}_{v} \gets \sum_{c \in \mathcal{C}}  \frac{w_c}{\mathcal{K}_c} W^{\mathcal{V}}_{v,c}$
		\EndFor
		
		\For {$i \in [\#Candidate Cameras]$}
		\State Select Voxel $v \sim W^{\mathcal{V}}$
		\State Sample Eye $e \sim \mathcal{U}(v)$
		\State Sample Gaze $g \sim \mathcal{A}$
		\State Camera $c_i \gets (e,g)$
		\State Image $I \in \mathcal{I} \gets$ Render(Scene,$c_i$)
		\State $W^{\mathcal{I}}_i \gets \sum_{(x,y,d,c) \in \mathcal{I}} f_c(x,y,d)$
		\For {$c \in \mathcal{C}$}
			\State $W^{\mathcal{I}}_{i,c} \gets \frac{w_c\sum_{(x,y,d,c') \in \mathcal{I}} f_c(x,y,d) \mathbbm{1}(c' = c)}{\sum_{(x,y,d,c') \in \mathcal{I}} \mathbbm{1}(c' = c)}$ 
		\EndFor
		\EndFor
		
		\State Order $I_i$ by descending $\sum_{c \in \mathcal{C}} W^{\mathcal{I}}_{i,c}$
		\State \Return $I_1...I_{threshold}$
		\EndFunction
		\Function{Select}{$I \in \mathcal{I}$, $W^{\mathcal{I}}_{I,c}$, $h_c$,$k$}
			\For {$c \in \mathcal{C}$}
				\State $F_{c} \gets \textrm{MinHeap()}$
				\For {$i \in [h_c]$}
					\State Push $0$ onto $F_{c}$
				\EndFor
			\EndFor
			\State $S \gets \{\}$
			\For {$I \in \mathcal{I}$}
				\State $\Delta(I | S) \gets \sum_{c \in \mathcal{C}} W^{\mathcal{I}}_{i,c}$
			\EndFor
			\For {$i \in [k]$}
				\State Order $\mathcal{I}$ by descending $\Delta(I | S)$ as $\mathcal{I}_1...\mathcal{I}_{|\mathcal{I}|}$
				\For {$i \in |\mathcal{I}|$}
					\State $\Delta(\mathcal{I}_i | S) \gets \sum_{c \in \mathcal{C}} \max(W^{\mathcal{I}}_{i,c} - F_{c}\mathrm{.top()}, 0)$
					\If{$\Delta(\mathcal{I}_i | S) > \Delta(\mathcal{I}_{i+1} | S)$}
						\For {$c \in \mathcal{C}$}
							\State Pop the top element from $F_c$ as $T$
							\State Push $\max(T,W^{\mathcal{I}}_{i,c})$ onto $F_c$
						\EndFor
						\State Move $\mathcal{I}_i$ from $\mathcal{I}$ to $S$.
						\State \bf{break}
					\EndIf
				\EndFor
			\EndFor
			\State \Return \it{S}
		\EndFunction
	\end{algorithmic} 
\end{algorithm}

\subsection{Candidate representation}

The first issue to address with this data-driven approach is to choose a suitable representation for a view.  Ideally, the representation should be small, while retaining the essential information about what makes a view representative of ones in an example set -- i.e., it should be both concise and informative.   Of course, we could compute any feature vector to represent a view in our framework, including a single number based on an embedding, features from a deep network, or all the pixels of a rendered image.  However, motivated by applications of semantic segmentation and scene understanding, we choose to represent each view $I \in \mathcal{I}$ as a $n$-dimensional vector $w_{I,1...n}$, where each dimension encodes information about the likelihood of a single semantic category's contribution to a rendered image for the view.  Since our examples come from the training set of NYUDv2, we select $n$ to be 40, where each value $w_{I,c}$ corresponds to an NYU40 category.  


More formally, we represent the contribution from a category $c \in \mathcal{C}$ to an image $I \in \mathcal{I}$ as the average likelihood of the pixels $p=(x,y,d,c') \in I$ where $c' = c$. Categories not present in an image are assigned a score of zero.  We compute the likelihood of an observation $(x,y,d)$ conditioned on class $c$ as $f_c(x,y,d)$, the value of the pdf of category $c$ defined over the three dimensional view volume. It represents the likelihood of observing category $c$ at pixel locations $x$ and $y$ and depth $d$ in the view volume relative to all other points in the view volume, and is normalized with $\ell_1$ normalization. We approximate each function $f_c(x,y,d)$ as a 3-dimensional histogram.  For each category $c$, we compute its contribution to $w_{I}$ for any view $I$ by looking up the likelihood of $f_c(x,y,d)$ for every pixel $(x,y,d)$ in $I$ and taking the average.

Intuitively this representation encodes the likelihood of observing a particular spatial distribution within each semantic category.  This general approach could be applied to other properties of scene observations or other data.  For RGB images, where depth is not available, we can use a two-dimensional histogram, which is equivalent to eliminating the uncertainty by integrating over the depth dimension. Alternatively, one could define $f_c(x,y,d) := f_c(x,y)p_c(d)$ where $p_c(d)$ is a specified prior over depth, i.e. for the semantic category sofa and the goal of generating NYUDv2-like images a gaussian centered 1.5 meters away might be reasonable.  The details of $f_c(x,y,d)$ do not affect the rest of the algorithm, as long as it is nonnegative and larger values represent more desirable contributions, which are requirements for submodular maximization.



\subsection{Candidate generation}

Our next step is to generate a discrete set of candidate views.  Given a CAD model of an indoor scene and a method to map any particular view $I$ to an $n$-dimensional vector $w$ describing the likelihood of its scene observations with respect to all $n$ semantic categories (as described in the previous subsection), the goal is to generate a set of candidate views $\mathcal{I}$ that will form the input to the view set selection algorithm in the next step.

Ideally, the output candidate set $\mathcal{I}$ will contain all views likely to be selected for the final output (high recall), but with as few extras as possible (high precision).  Finding the optimal set is not tractable: the space of views is infinite (with 6 continuous degrees of freedom representing translations and rotations), and evaluating the representation for any given view requires rendering an image and counting pixels in each semantic category.  Thus, we utilize approximations leveraging the specific properties of our view representation.  As our feature vector $w$ for each view measures its likelihood in the example distribution along each of several axes, we can approximate overall likelihood for a view or set of views by averaging the feature vectors over all in the set.   In particular, we can estimate the likelihood for particular view positions by averaging over all views at that position, marginalizing rotations and pixel locations ($x,y$).   This allows us to first sample view positions with high estimated likelihood and then later choose rotations, which saves having to render specific images in the first step.

Our algorithm first voxelizes the 3D scene, weighting the voxels $v \in \mathcal{V}$ in proportion to our estimates of the aggregate feature vector (that is, taking a weighted sum over the 40 values $W^{\mathcal{V}}_{v,:}$ in the approximation). The approximation at a voxel is computed by casting a set of rays into the scene and intersecting them with scene geometry; each ray provides a measured category $c$ and depth $d$ relative to the ray origin. The contribution of this sample to the aggregate feature vector is $f_c$ integrated over $x,y$, thus approaching in the limit the likelihood of the depth distribution of an image with maximal field of view at the ray center. Before taking the aggregate sum, weights may be applied on a per-category basis to further bias the search space. In our implementation, we use weights set to rebalance the category frequencies in the 3D scenes, SUNCG, to the frequencies in the example set, NYUDv2. 

Once voxels have been weighted in proportion to their estimated likelihood of selection, we generate candidate views in voxels with probability proportional to the voxel weight. To select a view direction, we do uniform sampling, except for the camera tilt. As a data-driven prior is also available for tilt in our chosen example dataset, NYUDv2, through accelerometer data, we exploit it by estimating the accelerometer distribution $\mathcal{A}$ and selecting tilts in proportion to that distribution. This distribution is not a major requirement of the algorithm proposed here: a gaussian prior can likely perform similarly.

The final step to candidate generation is a view filter. Since there are many obviously poor samples drawn from the view distribution, candidates from each room are first filtered to those with the highest aggregate scores. This is useful because candidates can be generated in parallel, while set selection is sequential. Thus, far more candidates can be generated than evaluated, and a method to improve the average candidate quality is useful for reducing runtime.

We run the candidate generation algorithm for all scenes in a synthetic training set, and union the outputs to form an overall candidate set for the view selection algorithm in the final step.

\subsection{View set selection}

The last step of our process is to select a subset of the candidate views that jointly reflect the example distribution.

Defining the goal in this task is tricky, because the true distribution of views is not known -- only example images are provided.   Assuming the example natural image dataset was generated according to some latent distribution, an ideal set would have high probability of having been generated by this latent distribution, without relying on replicating identical views to those in the example set.  More specifically, because we wish to have the capability to select a set that is far larger than the given dataset, we must avoid the pitfall of selecting images that individually have high probability of coming from the distribution, but as a set are unlikely because either: 1) they are all near the mean of the distribution, or 2) they form clusters around one or more of the example images without bridging the support space in between.  Our goal is to extend the distribution without replicating it. 

The first of the above concerns alone justifies the principle utility of a set-based distribution matching algorithm: it is impossible for any algorithm to determine from a single image whether it was generated according to a particular latent distribution, even if the distribution is known. Therefore, in principle any approach to viewpoint selection that uses only an embedding and generation scheme cannot match an adversarial latent distribution; a set-aware approach is required. 

The second of the above concerns precludes any method that selects or encodes image sets based on the distance to the nearest neighbor in the example distribution. In order to get around both of these challenges simultaneously, 
we formalize the problem of selecting a set $S$ from the set of candidate images $\mathcal{I}$ given their 40 dimensional likelihood estimates $W^{\mathcal{I}}_{I,:}$ with the following submodular objective (written for simplicity to assume that all images have distinct score vectors):
\vspace{-0.7em}
\begin{equation}\label{eq:1}
\begin{gathered}
\max_{S \subseteq \mathcal{I}} \sum_{c \in \mathcal{C}} W^{\mathcal{I}}_{I,c} \mathbbm{1}(     \exists_{V \subset S,|V| \geq |S|-h_c}\forall_{v\in V} W^{\mathcal{I}}_{v,c} < W^{\mathcal{I}}_{I,c}) \\\mathrm{s.t.} \ |S| \leq k
\vspace{-1em}
\end{gathered}
\end{equation}

The integer values $h_c$ represent the relative weights of the categories to the optimization, and should increase with the desired output count $k$; the higher this proportion, the less a single high likelihood image can cover the output space. In our implementation, we set $h_c := \frac{k}{40}$ for all categories, as the task we compare to is semantic segmentation and mean IOU is the target metric. We next prove that this optimization is tractable, even at the scale of millions of images, by exploiting the submodular maximization approaches described in Krause et al \cite{krause2012submodular}.

\begin{theorem}
	Equation \eqref{eq:1} is a nonnegative monotone submodular objective function.
\end{theorem}
\begin{proof}
	First, note that the function being maximized, which we refer to as $F(S) : S\subset \mathcal{I}\rightarrow \mathcal{R}$, is a summation over the weights $W^\mathcal{I}_{I,c}$. As those are nonnegative, the function itself is also nonnegative. Next, consider any two image sets $A$ and $B$, where $A \subset B \subseteq \mathcal{I}$. Take any image $I \in \mathcal{U} \setminus B$. We must show $\Delta(I | A) \geq \Delta(I | B)$. What can $I$ contribute to $\Delta(I | A)$? For any $c \in \mathcal{C}$, if the indicator function for the pair $I,c$ is positive, there are at most $h_c-1$ images $v \in V \subset A$ such that $W^\mathcal{I}_{v,c} > W^\mathcal{I}_{I,c}$. Order the weights $\{ W^\mathcal{I}_{v,c} : v \in A\}$ by decreasing value and name that ordering $W^{\mathcal{I},A}_{1,c},W^{\mathcal{I},A}_{2,c},...,W^{\mathcal{I},A}_{|A|,c}$. Similarly consider the same ordering for the corresponding category's weights for $B$: $W^{\mathcal{I},B}_{1,c},W^{\mathcal{I},B}_{2,c},...,W^{\mathcal{I},B}_{|B|,c}$. Then $W^{\mathcal{I},A}_{h_c,c} \leq W^{\mathcal{I},B}_{h_c,c}$, because $A \subset B$. So $W^{\mathcal{I}}_{I,c} - W^{\mathcal{I},A}_{h_c,c} \geq W^{\mathcal{I}}_{I,c} - W^{\mathcal{I},B}_{h_c,c}$. But these are the contributions of category $c$ to $\Delta(I | A)$ and $\Delta(I | B)$, respectively. Since this was done WLOG $c$, we have $\Delta(I | A) \geq \Delta(I | B)$.
	
	The final step is to show monotonicity. Consider again the proof that $\Delta(I | A) \geq \Delta(I | B)$. Substitute $B := A \cup \{e\}$ for any $e \in \mathcal{I}$. Then $\forall c \in \mathcal{C}, i \in [|A|] \ W^{\mathcal{I},A}_{i,c} \leq W^{\mathcal{I},A \cup \{e\}}_{i,c}$, implying $F(A) \leq F(A \cup \{e\})$.
\end{proof}
 
Since we have shown our objective is nonnegative, monotone, and submodular, there is a greedy $\frac{1}{2}$-approximation algorithm due to Nemhauser et al \cite{nemhauser1978analysis}.  The algorithm is adapted to this problem by computing the discrete derivative $\Delta(I | S)$ for a candidate image as the sum over categories $c$ of the marginal gain of $W^\mathcal{I}_{I,c}$ over the least quantity in $V$ still contributing to the optimization. We use a set of minheaps to maintain the sets quickly in practice, where each minheap always contains exactly $h_c$ values. The integer weights $h_c$ per category are selected by dividing the image budget $k$ with image counts evenly amongst all categories, though unbalanced weights, such as those derived by rounding from the relative frequencies of the categories in NYUDv2 training, are also possible.

\section{Experiments}

\begin{figure*}
	\vspace{-1.5em}
	\input{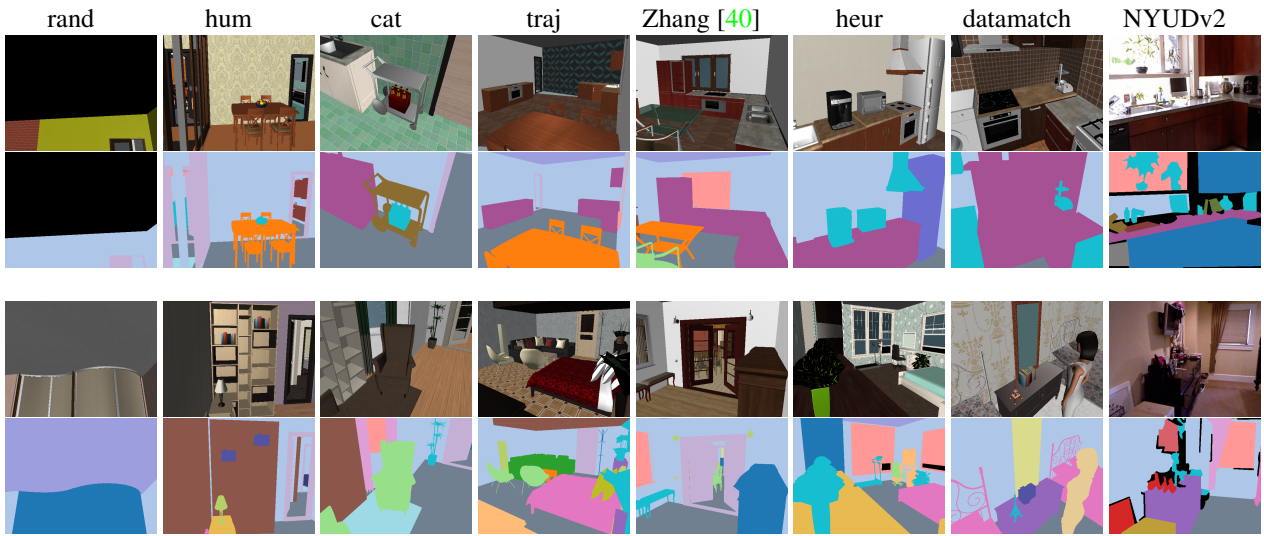}
	\caption{Each column shows viewpoints selected from SUNCG 3D scenes using a different selection algorithm (last column shows an example NYUDv2 image from the same room category).  The first and third rows show color, and the second and fourth show NYU40 category semantic segmentations.}
	\label{fig:view_examples}
	\vspace{-1em}
\end{figure*}

In this section we describe a series of experiments to evaluate how well different viewpoint selection algorithms match the statistics of real-world data, and how much difference view selection strategies can make on training performance for semantic segmentation on RGBD data.

\mypara{Datasets.}
For all our experiments, we use SUNCG~\cite{song2017ssc} as the source of synthetic 3D scenes to render.
We select a set of 171,496 rooms from the SUNCG scenes, by filtering for rooms that have a floor area of 8 to 100 squared meters, with at least four walls and five objects present.
We choose NYUDv2~\cite{silberman2012indoor} as our target data for our data matching algorithm and for evaluating semantic segmentation performance.
This allows us to test the benefit of both synthetic depth and synthetic color information.
We use the NYU40 category set as defined by Gupta et al.~\cite{gupta2014learning} for semantic segmentation and for computing all category-level statistics.
The standard NYUDv2 split of 795 training images and 654 testing images is used for our algorithm's data matching and the semantic segmentation evaluation.

\subsection{Comparison algorithms}

We compare images generated using variants of our algorithm against a variety of alternative methods for viewpoint selection including random cameras, single object closeup views, human-like navigation trajectories, and manual heuristics from prior work.  For each algorithm, we generate a set of twenty thousand camera views (see \Cref{fig:view_examples} for examples).  We describe the different methods below.

\mypara{Random cameras (\rand).}
For each room, a single random view is computed. A camera position is randomly sampled within the room volume, rejecting samples inside the axis aligned bounding box of scene objects. The view direction is randomly sampled uniformly.  This is similar to the approach described by Handa et al.~\cite{handa2015scenenet} except that they select viewpoints with at least 3 objects.

\mypara{Random human eye level (\hum).}
In addition to the purely random view selection, we also consider a baseline algorithm where we restrict the camera positions to be at an average eye height of 1.55 meters, corresponding to holding a camera just below eye height.

\mypara{Object closeups (\cat).}
The category-based camera approach focuses on spending the image budget equally across the NYU40 categories.  Four NYU40 categories not present in SUNCG, and the wall, ceiling, and floor categories are not allocated a budget. The remaining 33 categories are each allocated approximately $1/33$rd of the 20,000 images. For each image of a given category, an object is chosen with replacement from the set of all objects of that category. For each chosen object, an image focusing on that object from eye height is selected. A set of candidate images in a ring in the eye height plane around the object are generated, and the image with the greatest fraction of pixels belonging to the object is selected.

\mypara{Trajectories (\traj).}
Trajectory cameras are selected from an unobstructed path in the room. A two dimensional positional grid is imposed on each room, and grid points are flagged as obstructed if they are inside or close to the surface of an object. For each such room, a point is selected. The furthest grid point from that point under an unobstructed distance metric is selected as one endpoint of the trajectory. The second endpoint is selected to be the furthest point from that point, again constraining the candidate set to paths not within the obstacles of the room. Images are densely sampled from the trajectory. For each image, the view direction is towards the centroid of object centroids, weighted by the number of faces of each object. Views are subsampled randomly to reach 20K images.  At a high level, \traj is an approach similar to the two-body trajectory algorithm used by McCormac et al.~\cite{mccormac2016scenenet}.

\mypara{Zhang et al.~\cite{zhang2016render}.}
We compare against the views generated by Zhang et al.~\cite{zhang2016render}.  Their method scores views based on object count and pixel coverage, and filters raytraced images using a color and depth histogram similarity to NYUDv2 images.  The authors provided their generated view set which we subsampled randomly to match the total set size used for all other methods.

\mypara{Hand-tuned heuristic (\heur).}
We implement a hand-tuned heuristic that maximizes the number and pixel occupancy of object categories in the output views.  First, views are densely sampled at eye height in a grid of unobstructed positions in the room, sampling several view directions per position, each of which is given a fixed downward tilt.  Each sampled view is assigned a score $s = \sum_{i \in v_i}{\mathbb{1}(p_i > m) (\log(p_i)-\log(m))}$ where $p_i$ is the number of pixels of object $i$ in view $v_i$, and $m$ is a minimum number of pixels per object.  The top scoring images across all rooms are chosen as the training set.  Intuitively, this method attempts to maximize the salient information in the frame, and is an attempt at improving Zhang et al.'s method.

\mypara{Data matching (\datamatch).}
We apply our data-driven view set selection algorithm to match RGBD frames from the NYUDv2~\cite{silberman2012indoor} training images, in terms of the distribution within-frame of the NYU40 categories.  We rebalance the category frequencies in SUNCG to the frequencies of the NYU40 categories.

\begin{table}
\centering
\begin{tabular}{lcccc}
\toprule
method & depth & x & y  & mean \\
\midrule
\rand & 0.56 & 0.166 & 0.474 & 0.400\\
\hum & 0.404 & 0.169 & 0.244 & 0.272\\
\cat & 0.434 & 0.177 & 0.192 & 0.268\\
\traj & \textbf{0.346} & 0.107 & 0.145 & 0.199\\
\yinda & 0.377 & 0.100 & 0.212 & 0.230\\
\heur & 0.349 & 0.104 & 0.183 & 0.212\\
\datamatch & \textbf{0.346} & \textbf{0.094} & \textbf{0.134 }& \textbf{0.191}\\
\bottomrule
\end{tabular}
\caption{EMD with thresholded ground distances for each axis of object occurence, averaged over the NYU40 categories. Distances above are normalized to [0-1], where lower values are closer to NYUDv2.}
\vspace{-1.2em}
\label{table:emd}
\end{table}

\subsection{Evaluation with earth mover's distance}

We perform three types of experiments to investigate characteristics of our view set selection algorithm.  The first is focused on quantitative evaluation of how well it generates distributions of images matching an example set.  To address this question, we use an earth mover's distance (EMD) to measure the difference between two RGBD image sets.  Specifically, we use the thresholded distance metric as described in Pele et al.~\cite{pele2009fast} as it is particularly suitable for image data. For each category in an image set, we generate three separate 1D histograms: the $x$, $y$, and depth occurrences of that category throughout the set. For each such histogram, the EMD to the corresponding histogram from the training \& validation set of NYUDv2 is computed, and the results are averaged over all 40 NYUDv2 categories.

\Cref{table:emd} shows the results.  As we can see, the \datamatch algorithm achieves the lowest EMD to NYUDv2 frames overall.  The \traj algorithm performs second best, most likely due to the fact that NYUDv2 was collected by following human-feasible trajectories through real rooms.  In any case, this experiment confirms that our algorithm can produce distributions of images more similar to an example set than alternatives.

\begin{table*}[t]
\vspace{-1em}
\centering
\begin{tabular}{lcccc|cccc}
\toprule
 & \multicolumn{4}{c|}{HHA} & \multicolumn{4}{c}{RGB}\\
method & Pixel Acc. & Mean Acc. & Mean IoU & F.W. IoU & Pixel Acc. & Mean Acc. & Mean IoU & F.W. IoU \\
\cmidrule(lr){2-5}
\cmidrule(lr){6-9}
FCN-32s~\cite{long2015fully}        & 58.2 & 36.1 & 25.2 & 41.9 & 61.8 & 44.7 & 31.6 & 46.0 \\
\midrule
\rand                               &  6.1 &  3.1 &  0.6 &  1.3 & 59.3 & 41.0 & 28.4 & 44.2 \\
\hum                                & 60.0 & 39.8 & 28.1 & 44.3 & 61.4 & 42.2 & 30.0 & 45.7 \\
\cat                                & 61.3 & 41.6 & 29.8 & 45.5 & 62.6 & 44.5 & 31.6 & \textbf{46.9} \\
\traj                               & 60.7 & 40.9 & 28.6 & 44.7 & 61.2 & 42.5 & 30.2 & 45.1 \\
\yinda{\color{red} $^1$}                   & 61.3 & 41.9 & 29.8 & 45.5 & \textbf{62.7} & 43.7 & 31.8 & 46.3 \\
\heur                               & \textbf{61.7} & 42.2 & 30.1 & \textbf{45.9} & 62.6 & 44.5 & 31.7 & \textbf{46.9} \\
\datamatch                          & 61.2 & \textbf{42.7} & \textbf{30.4} & 45.1 & 62.4 & \textbf{45.1} & \textbf{32.1} & 46.8 \\
\bottomrule
\end{tabular}
\caption{{\bf Semantic segmentation results on the NYUDv2 test set after pretraining on different synthetic image datasets.}}
\label{table:seg_results}
\vspace{-1em}
\end{table*}

\subsection{Semantic segmentation results}

In our second experiment, we investigate the impact of our view set selection approach on the performance of pretrained networks for semantic segmentation.

\mypara{Procedure.}
We follow the training procedure of Zhang et al.~\cite{zhang2016render} to train a fully convolutional network~\cite{long2015fully}.
Initial weights are taken from FCN-VGG16, the weights of VGG 16-layer net~\cite{simonyan2015very} pretrained on ImageNet~\cite{deng2009imagenet} converted to an FCN-32s network.
We then pretrain this network using the images rendered from a given algorithm's 20k selected views in SUNCG for five epochs, and finally we finetune on the NYUDv2 training set for 100K iterations.
We run experiments using depth or color images.  For depth, we use the HHA depth encoding of Gupta et al.~\cite{gupta2014learning}.  For color images, we render an RGB image for each view using the OpenGL rendering toolkit provided by Song et al.~\cite{song2017ssc}.
We run all our experiments with the Caffe library~\cite{jia2014caffe} and the public FCN implementation~\cite{long2015fully}, on Nvidia Titan X GPUs.

\mypara{Results.}
\Cref{table:seg_results} reports weighted and unweighted accuracy and IoU metrics for the NYUDv2 test set, as defined by Long et al.~\cite{long2015fully}.  We see that good view selection makes a big difference for depth with \datamatch views improving the baseline FCN HHA mean IoU by 5.2 points (20.6\% relative improvement).  A smaller impact is seen for synthetic color, where \datamatch leads to a 0.6 point improvement (1.9\% relative improvement).  We hypothesize this is due to the large color disparity between the synthetic color images and the real world color.  For depth, using the HHA encoding, \heur and \datamatch both outperform the camera selection algorithm used in \cite{zhang2016render}\footnote{Reported HHA performance in Zhang et al. is significantly lower at 23.4\% mean IoU.}.  The \rand algorithm performs especially bad since the HHA encoding relies on a built-in assumption of a dominant front orientation of the camera.  Moreover, \hum, \traj, and \rand significantly reduce baseline performance in color indicating that pretraining with bad synthetic views can hurt instead of help.

\mypara{Influence of submodular optimization for set selection.}
To isolate the impact of the submodular optimization part of our algorithm we also compare against a version of \datamatch that only selects the best $k$ scoring images instead.  The views generated by this ablated version of the algorithm result in reducing performance on HHA to 61.8 pixel accuracy, 42.2 mean accuracy, 30.2 mean IoU, and 45.8 frequency weighted IoU.




\subsection{Human view selection judgments}


In a third experiment, we investigate the correlation between semantic segmentation performance and human view selection by running a study where we present people with a target set of NYUDv2 images in a given room type (e.g., bedroom) and ask them to select from a random set of five synthetic views sampled uniformly from the images of all seven algorithms considered in the previous experiments.  Our goal in this experiment is to characterize whether human judgement of good views is as good for semantic segmentation as our proposed algorithm.

We recruited 261 Amazon Mechanical Turk workers to select any number of synthetic images that match the types of views of the target NYUDv2 set.  Each worker selected images from a total of 30 random sets of five images.  Overall, the workers saw 27,564 images and selected 10,046 as matching the NYUDv2 target set.  To evaluate these results, we measure the fraction of time that an image from a given algorithm was selected as matching.  The mean values are \rand: 0.04, \hum: 0.29, \cat: 0.34, \traj: 0.51, \yinda: 0.50, \heur: 0.53, \datamatch: 0.41.  Interestingly, views produced by algorithms that provide greater EMD results and lesser semantic segmantion results than \datamatch are selected more often by people.  This phenomenon suggests that people might not be a good judge of how images fit into distributions, corroborating the need for automatic algorithms for this task.

\begin{figure}
	\centering
	\includegraphics[width=\linewidth]{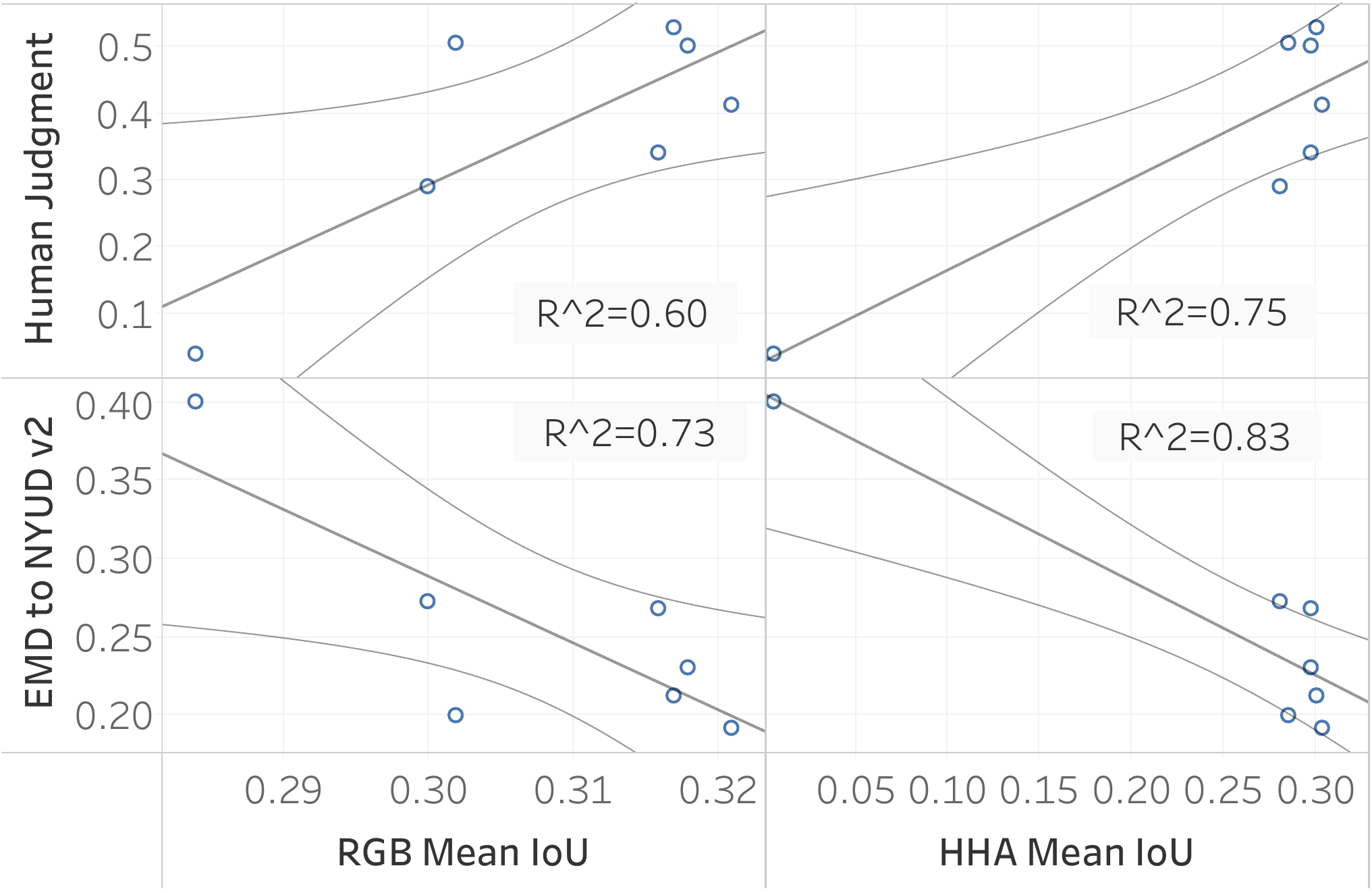}
	\caption{Semantic segmentation performance (left: RGB, right: HHA) plotted against both EMD distance of images to NYUDv2 and subjective human judgment of how close images are to the set of NYUDv2 images.}
	\label{fig:correlations}
\end{figure}

We compare the human judgments of match to NYUDv2 against EMD as predictors of semantic segmentation performance.  \Cref{fig:correlations} plots each algorithm as a point and shows the linear fit (along with 95\% confidence contours), and the $R^2$ measure of goodness of fit.  As expected, EMD is inversely correlated with segmentation performance, and human judgments are mostly positively correlated with performance.  The higher $R^2$ numbers for EMD-to-segmentation fit indicate that EMD is a better predictor of performance than human judgment.  The higher correlation of EMD to semantic segmentation performance suggests that our choice of image representation is effective.


\section{Conclusion}

In summary, this paper has investigated the idea of using data-driven
algorithms to select a set of views when creating synthetic datasets
of rendered images for pre-training deep networks.  It provides a formulation
for representing views based on pdfs of object observations of specific
semantic categories, generating candidate views with an efficient and
practical search algorithm, and selecting a set of views matching a
target distribution with submodular maximization.   Results of experiments
show that careful view select with this approach can improve the performance
of semantic segmentation networks.

This is just the tip of the iceberg on how to select views for
training deep networks for computer vision tasks.  One way to improve our algorithm
is to remove the independence assumption between categorical likelihoods
(i.e., explicitly model object co-occurence probabilities).  We chose a formulation
motivated by semantic segmentation, but others may be better suited for different
tasks.  We establish that view set selection is an important problem in this domain
and conjecture that much future work is warranted.

{\small
\bibliographystyle{ieee}
\bibliography{hcv}
}\newpage

\appendix
\section*{Appendix}
\section{Example Viewpoint Sets}

The viewpoint set selection problem we pose in our paper takes as input a 3D scene and provides as output a set of viewpoints in that scene that best match a prior data distribution.  \Cref{fig:viewsets} shows two example 3D scenes and output viewpoint sets selected using our \datamatch algorithm.

\begin{figure*}
	\includegraphics[width=\linewidth]{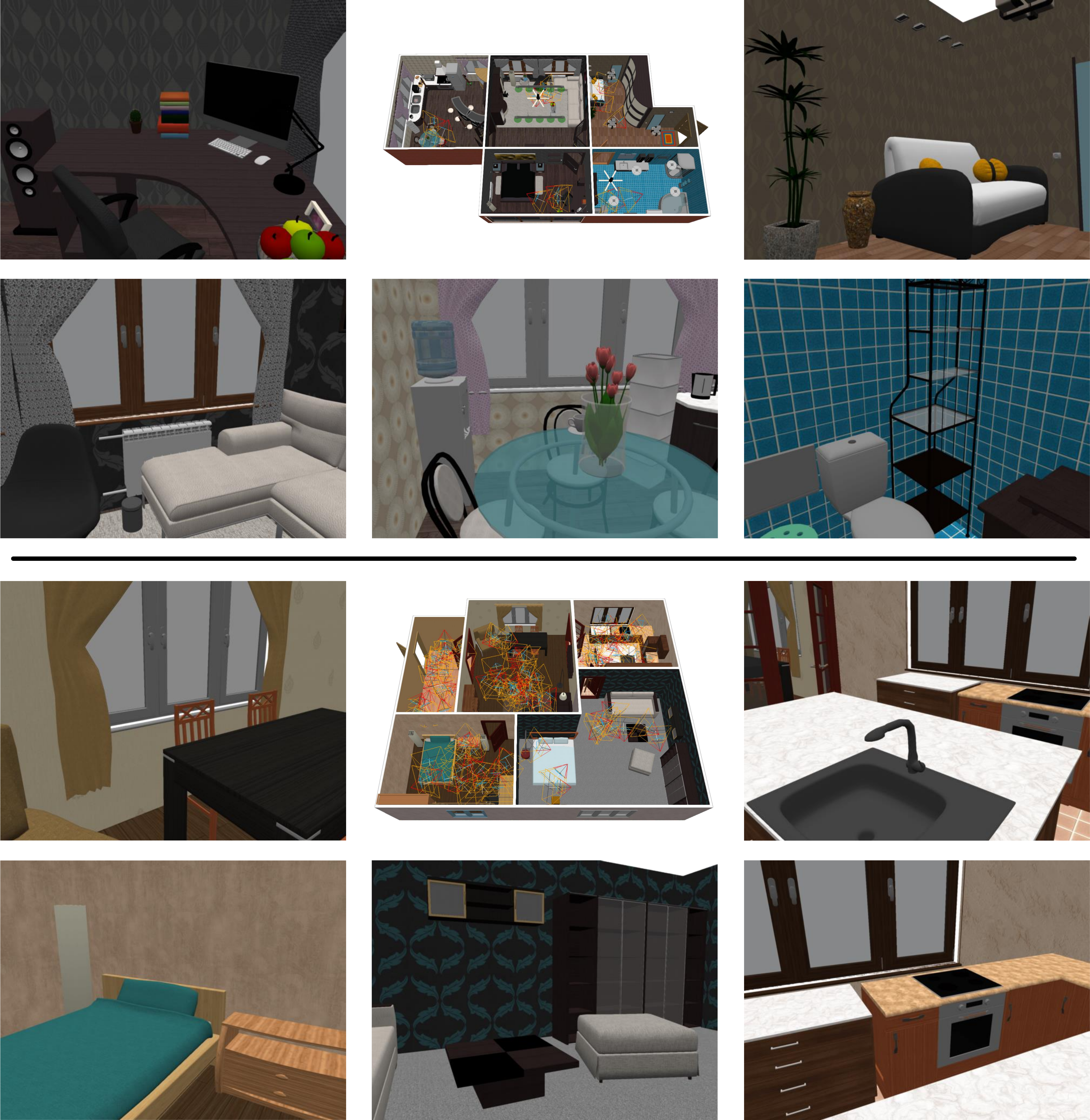}
	\caption{Two example 3D scenes and their viewpoint sets selected by the \datamatch algorithm trained on NYUDv2.  Viewpoints are shown in the overview images as view frustum cones drawn in red and orange outlines.  Corresponding images for a subset of the viewpoints in each are shown around the scene.  For the first scene, the top 3 views for each room were selected in the set, whereas for the second scene, the top 20 were selected.}
	\label{fig:viewsets}
\end{figure*}

\section{Single Image Matching}

The \datamatch algorithm was designed to match the object distributions seen in sets of images, and to generate sets of viewpoints that exemplify similar distributions.  However, it can be directly applied to single image inputs and restricted to generate one viewpoint that best matches the input, given a particular target 3D scene.  \Cref{fig:single_im} shows some example viewpoints generated to match a particular view of a kitchen counter.

\begin{figure*}
	\centering
	\begin{subfigure}[t]{0.45\linewidth}
		{%
		\setlength{\fboxsep}{2pt}%
		\setlength{\fboxrule}{2pt}%
		\fbox{\includegraphics[width=0.97\linewidth]{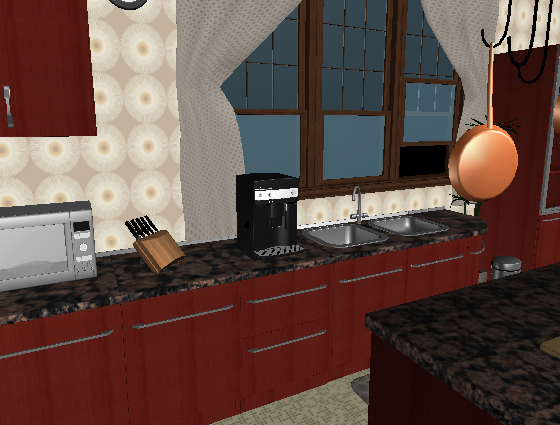}}%
		}
	\end{subfigure}
	\quad
	\begin{subfigure}[t]{0.45\linewidth}
		\includegraphics[width=\linewidth]{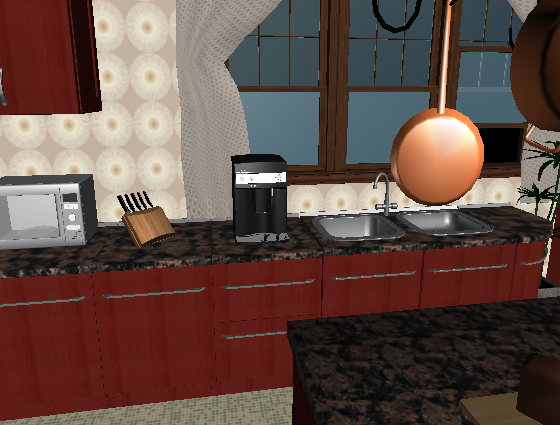}
	\end{subfigure}
	\vskip\baselineskip
	\begin{subfigure}[t]{0.45\linewidth}
		\includegraphics[width=\linewidth]{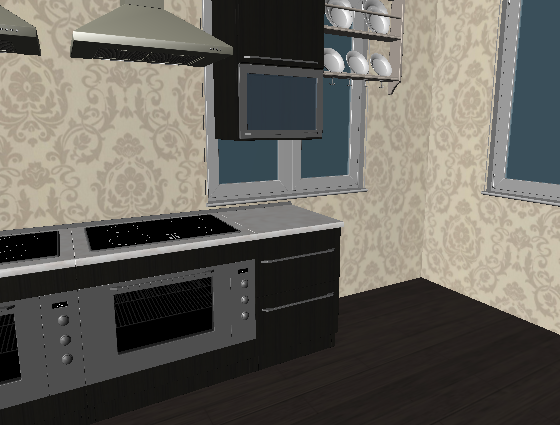}
	\end{subfigure}
	\quad
	\begin{subfigure}[t]{0.45\linewidth}
		\includegraphics[width=\linewidth]{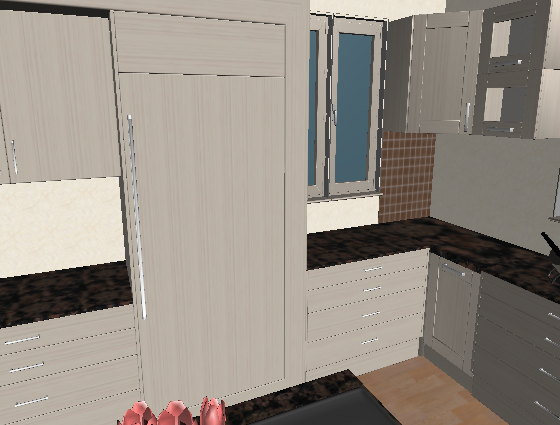}
	\end{subfigure}
	\vskip\baselineskip
	\begin{subfigure}[t]{0.45\linewidth}
		\includegraphics[width=\linewidth]{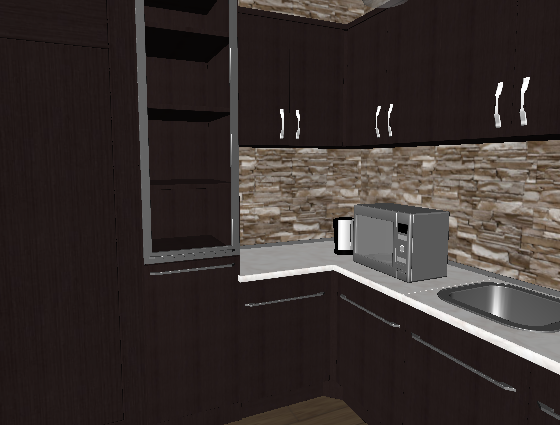}
	\end{subfigure}
	\quad
	\begin{subfigure}[t]{0.45\linewidth}
		\includegraphics[width=\linewidth]{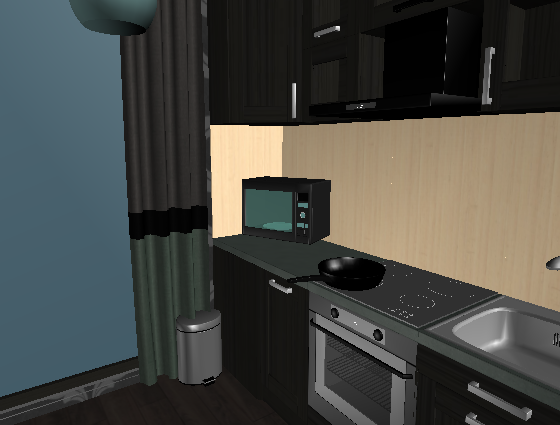}
	\end{subfigure}
	\caption{\textbf{Top left}: input image exemplifying desired framing.  \textbf{Top right}: output viewpoint in same room, generated using the \datamatch algorithm.  \textbf{Middle and bottom}: output viewpoints in other rooms.  Note that the objects seen and their framing in the view is similar to the input image, despite differences in the room layout.}
	\label{fig:single_im}
\end{figure*}

\section{RGB Image Matching}

A key strength of the \datamatch algorithm we present is that it can be directly applied to match viewpoint distributions in arbitrary target datasets.  Moreover, though the main paper evaluation targets the RGBD data in NYUDv2, using depth data is not a requirement of the algorithm.  We demonstrate how our algorithm generalizes to a viewpoint distribution that differs drastically from the views in NYUDv2.  To this end, we collected a small dataset of semantically annotated images from an RGB camera mounted on a Roomba robot and generated a viewpoint set for a living room SUNCG scene.  \Cref{fig:roombacam} shows the results. To replace the voxel weighting step, uniform voxel weights were used for all voxels in the 3D room volume; the sample count was not increased.

\begin{figure*}
	\centering
	\begin{subfigure}[t]{0.15\linewidth}
		\includegraphics[width=\linewidth]{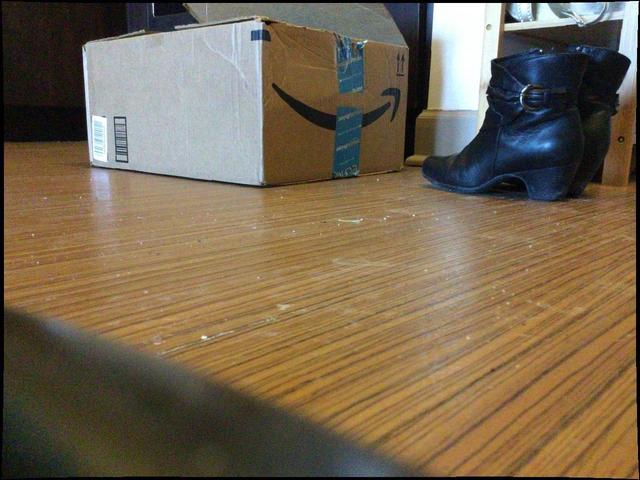}
	\end{subfigure}
	\begin{subfigure}[t]{0.15\linewidth}
		\includegraphics[width=\linewidth]{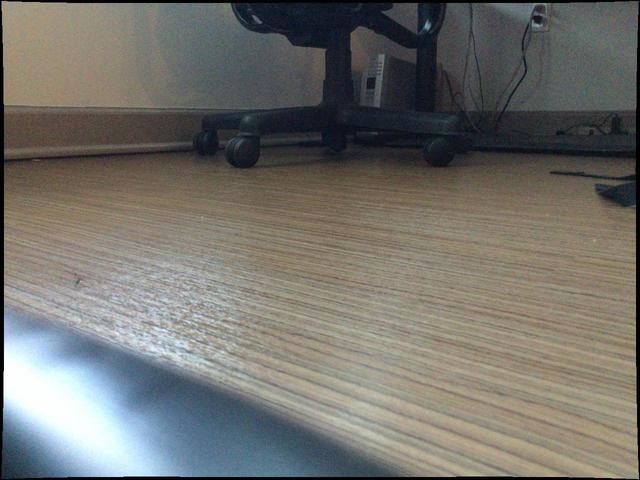}
	\end{subfigure}
	\begin{subfigure}[t]{0.15\linewidth}
		\includegraphics[width=\linewidth]{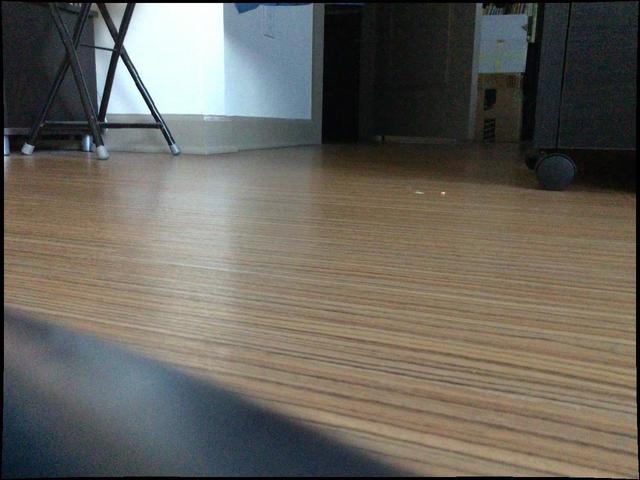}
	\end{subfigure}
	\begin{subfigure}[t]{0.15\linewidth}
		\includegraphics[width=\linewidth]{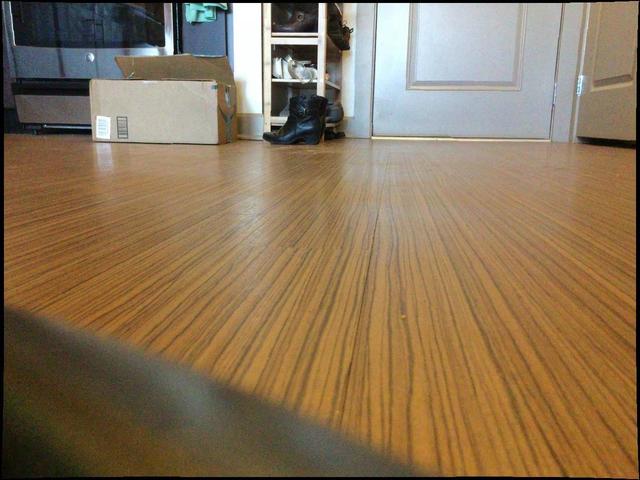}
	\end{subfigure}
	\begin{subfigure}[t]{0.15\linewidth}
		\includegraphics[width=\linewidth]{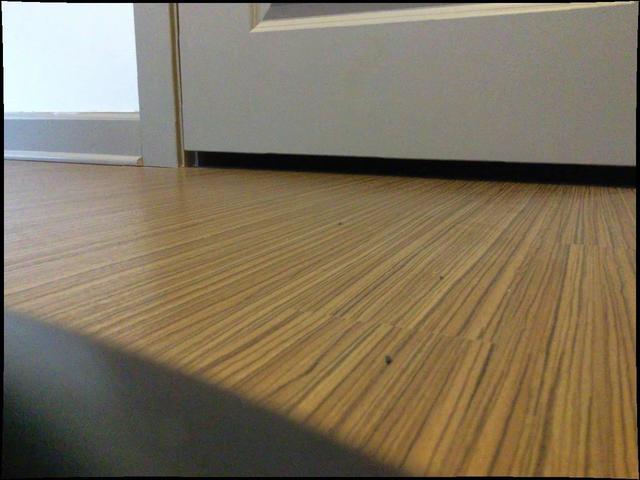}
	\end{subfigure}
	\begin{subfigure}[t]{0.15\linewidth}
		\includegraphics[width=\linewidth]{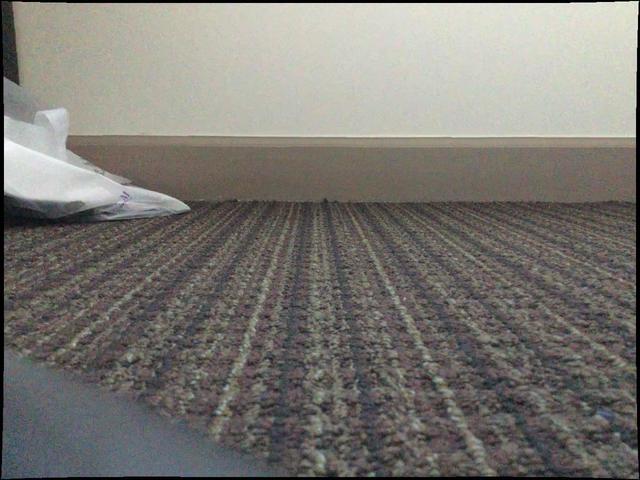}
	\end{subfigure}
	\\
	\begin{subfigure}[t]{0.15\linewidth}
		\includegraphics[width=\linewidth]{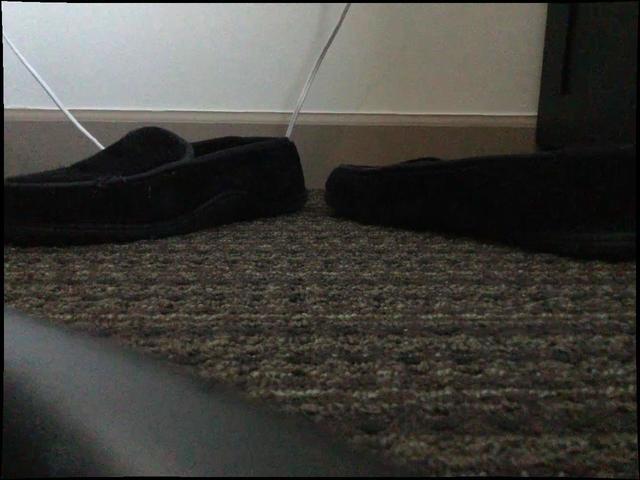}
	\end{subfigure}
	\begin{subfigure}[t]{0.15\linewidth}
		\includegraphics[width=\linewidth]{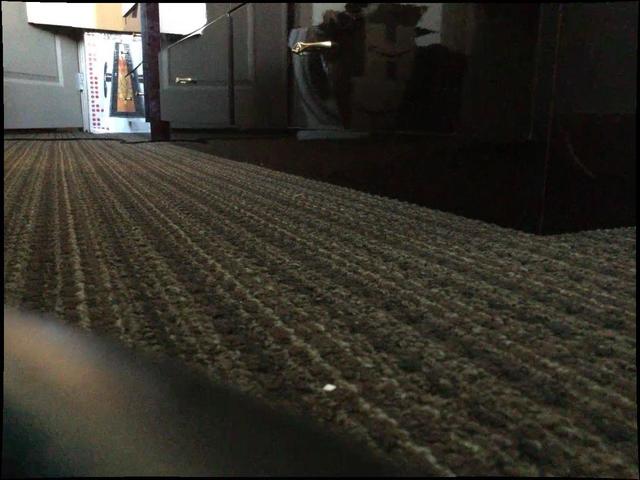}
	\end{subfigure}
	\begin{subfigure}[t]{0.15\linewidth}
		\includegraphics[width=\linewidth]{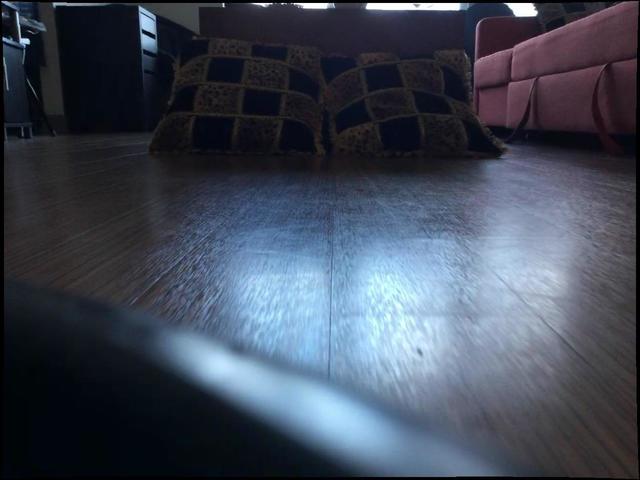}
	\end{subfigure}
	\begin{subfigure}[t]{0.15\linewidth}
		\includegraphics[width=\linewidth]{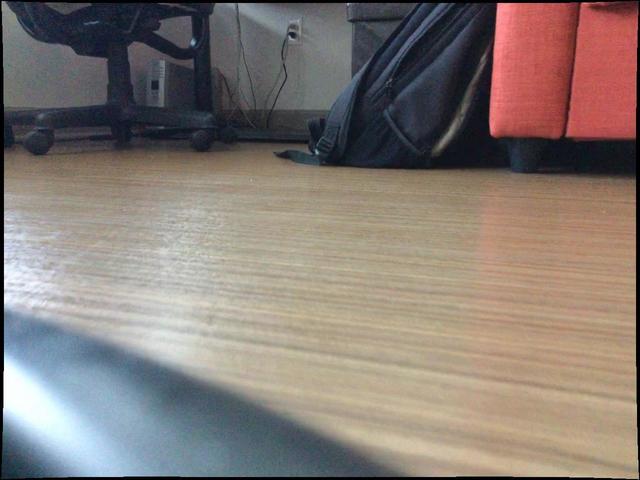}
	\end{subfigure}
	\begin{subfigure}[t]{0.15\linewidth}
		\includegraphics[width=\linewidth]{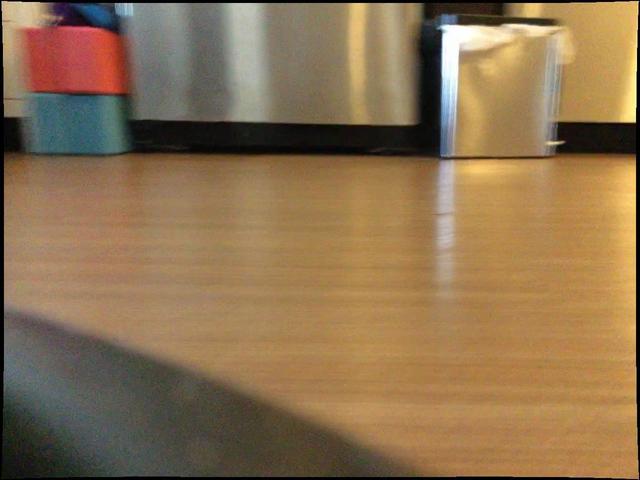}
	\end{subfigure}
	\begin{subfigure}[t]{0.15\linewidth}
		\includegraphics[width=\linewidth]{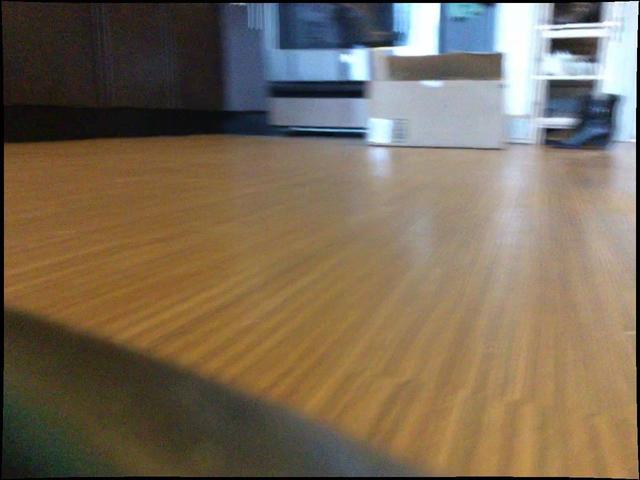}
	\end{subfigure}
	\vskip\baselineskip
	\vskip\baselineskip
	\begin{subfigure}[t]{0.46\linewidth}
		\includegraphics[width=\linewidth]{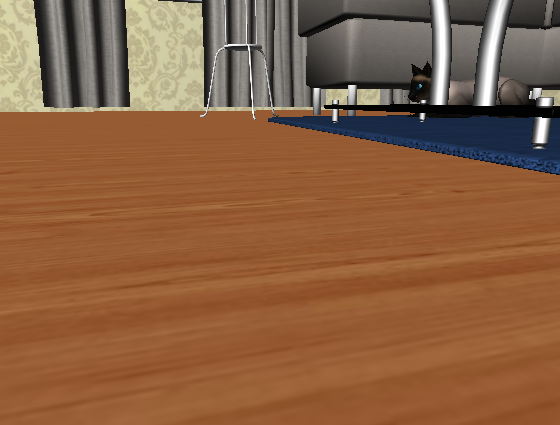}
	\end{subfigure}
	\begin{subfigure}[t]{0.46\linewidth}
		\includegraphics[width=\linewidth]{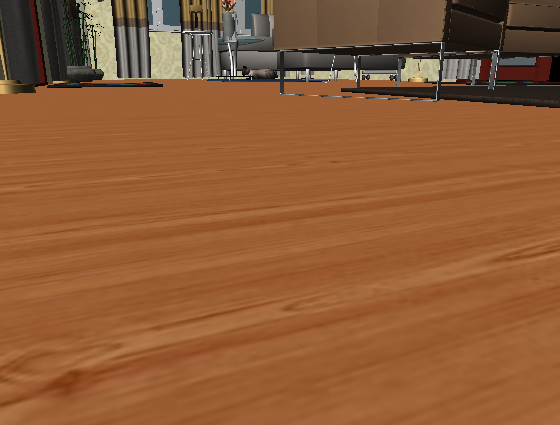}
	\end{subfigure}
	\\
	\begin{subfigure}[t]{0.46\linewidth}
		\includegraphics[width=\linewidth]{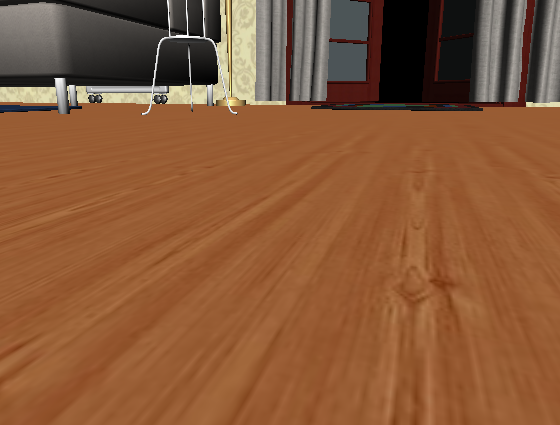}
	\end{subfigure}
	\begin{subfigure}[t]{0.46\linewidth}
		\includegraphics[width=\linewidth]{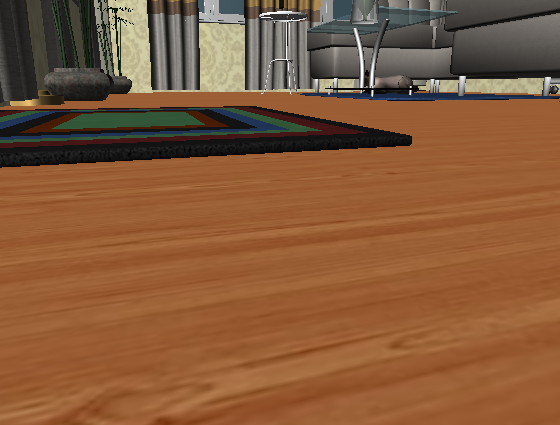}
	\end{subfigure}
	\caption{\textbf{Top}: RGB images collected using an iPad Air2 camera attached to a Roomba robot.  We semantically annotated these images and used the \datamatch algorithm to generate matching viewpoints in a 3D scene.  \textbf{Bottom:} Generated viewpoints which exhibit the same view distribution as the input data.  Note that depth data was not used, and the algorithm was not changed in any way to account for differences between Roomba viewpoints and human viewpoints.}
	\label{fig:roombacam}
\end{figure*}

\section{User Study Details}

\Cref{fig:user_study} shows screenshots of the interface that we used to carry out the user study asking people to select generated views that match the NYUDv2 viewpoints.  Each participant saw 30 randomly sampled sets of viewpoints, and matched them against a randomly sampled set of NYUDv2 images from the same room category.

\begin{figure*}
	\centering
	\begin{subfigure}[t]{0.8\linewidth}
		\includegraphics[width=\linewidth]{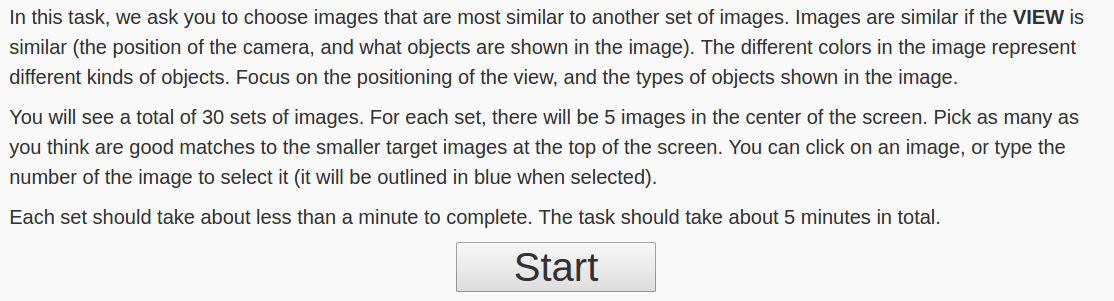}
	\end{subfigure}
	\\
	\begin{subfigure}[t]{0.8\linewidth}
		\includegraphics[width=\linewidth]{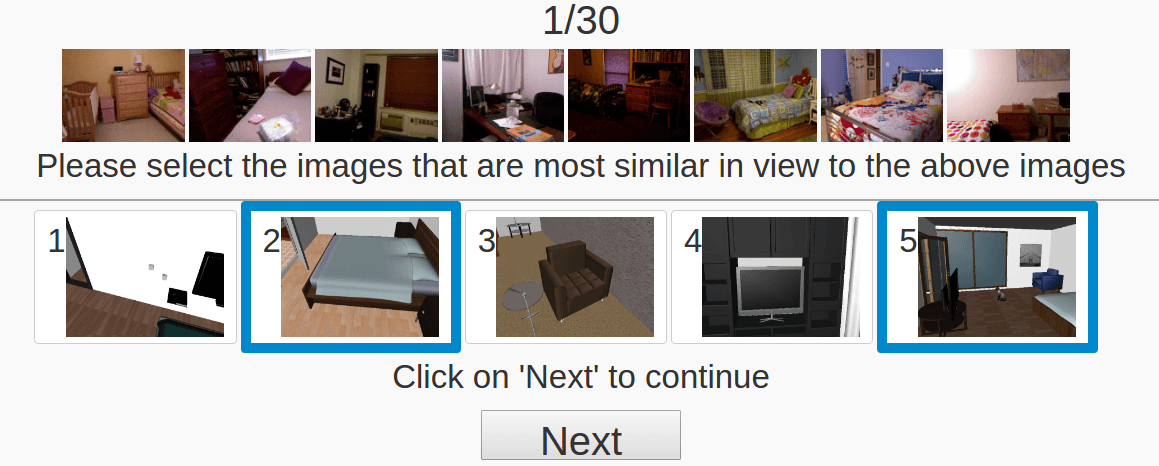}
	\end{subfigure}
	\\
	\begin{subfigure}[t]{0.8\linewidth}
		\includegraphics[width=\linewidth]{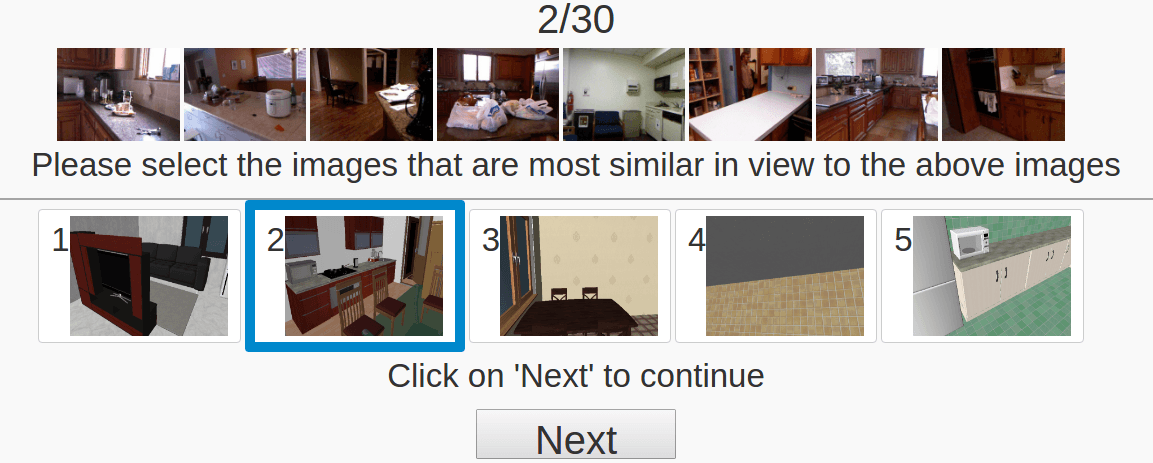}
	\end{subfigure}
	\\
	\begin{subfigure}[t]{0.8\linewidth}
		\includegraphics[width=\linewidth]{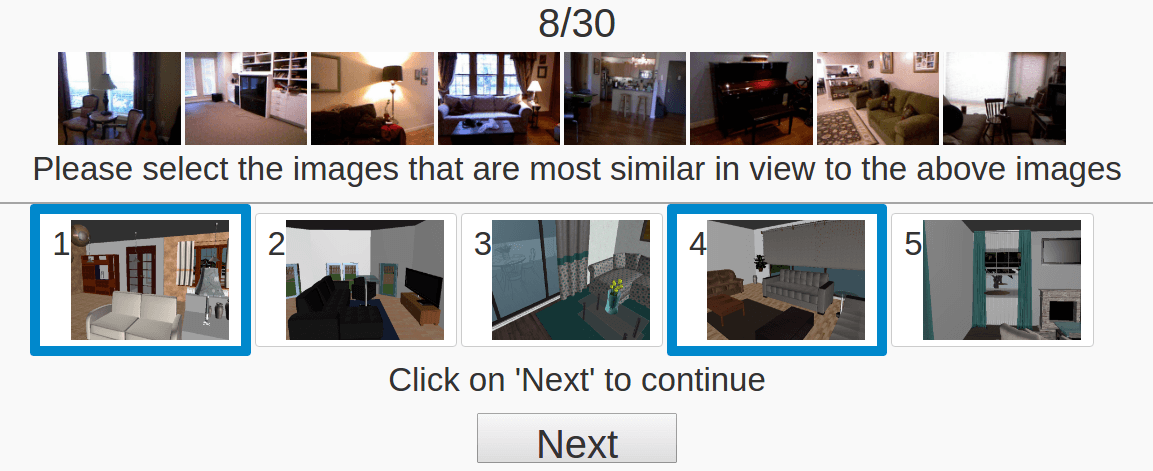}
	\end{subfigure}
	\caption{Screenshots of the instructions and interface for the view set match selection user study we carried out.  Participants were instructed to select generated viewpoints images in the bottom row that match the viewpoints of the top example images from NYUDv2.}
	\label{fig:user_study}
\end{figure*}

\section{Constants}

The 40 weights $w_c$ per category were chosen to approximately rebalance SUNCG frequencies to NYUDv2 frequencies, and are given in Table \ref{table:catweights}. The rebalancing was object occurence in NYUDv2 divided by object occurence in SUNCG. Weights marked N/A correspond to categories not present in SUNCG, and the wall, ceiling, and floor categories were set to a weight of 1.  The candidate generation step used 1500 candidates per room, 200 raycasts per voxel, and filtering down to 20 candidates per room. These settings resulted in roughly comparable runtime to the \heur algorithm.

When generating images from only a single image or handful of images as is the case in this supplemental document, the constants and \datamatch algorithm are adjusted to compensate for the increased noise in the semantic category pdfs. First, for each exemplifying image in the input, a univariate gaussian with standard deviation equal to 10\% of the histogram resolution is added to each semantic category's pdf along the depth axis, rather than a single point per pixel. Next, to compensate for the narrower distribution being matched, the candidate sample count is increased to 3500 to ensure good samples are still found. The candidate image vector likelihood values are not divided by the observation frequency; this helps stabilize filtering when the input image count is small. Finally, as the distribution is no longer intended to match NYUDv2, uniform category weights $w_c$ are used. Note also that as the images generated for the figures here are created from a single room's output rather than that of 20,000 rooms, the SELECT portion of the \datamatch algorithm is superfluous.

\begin{table*}
	\centering
\begin{tabular}{||lr||}
	\hline
wall & 1.0000\\
floor & 1.0000\\
cabinet & 1.0770\\
bed & 0.6267\\
chair & 0.8133\\
sofa & 0.7622\\
table & 0.7552\\
door & 0.1293\\
window & 0.2678\\
bookshelf & 6.3694\\
picture & 2.3617\\
counter & 0.2667\\
blinds & 1.6238\\
desk & 0.6391\\
shelves & 0.4318\\
curtain & 0.4273\\
dresser & 0.5381\\
pillow & 11.8598\\
mirror & 0.7994\\
floor\_mat & 0.2805\\
clothes & 38.6042\\
ceiling & 1.0000\\
books & 4.0946\\
refridgerator & 0.5520\\
television & 0.1626\\
paper & N/A\\
towel & N/A\\
shower\_curtain & 0.1942\\
box & N/A\\
whiteboard & 1.6733\\
person & 0.3813\\
night\_stand & 0.2648\\
toilet & 0.1204\\
sink & 0.4990\\
lamp & 0.1518\\
bathtub & 0.2869\\
bag & N/A\\
otherstructure & 2.8287\\
otherfurniture & 497.8754\\
otherprop & 1.1364\\
\hline
\end{tabular}
\caption{Category weights $w_c$ for the NYU40 categories.  These weights were chosen to rebalance the category occurrence frequencies of SUNCG to NYUDv2 images.}
\label{table:catweights}
\end{table*}

\section{Runtime}

The runtime of the \datamatch algorithm is comparable to that of other methods. When run on 20,000 selected scenes in SUNCG, the total candidate generation time was 898 hours, 37 minutes; this was run in parallel on a cluster with approximately 500 simultaneously allocated cores. In addition, the set selection stage required 22 minutes, 11 seconds. By comparison, the total runtime for the \heur method on the same dataset was 880 hours, 54 minutes. For both algorithms, the primary computational bottleneck is rendering candidate images in order to score them. Note that the times reported above do not include rendering the selected images to disk for training, as OpenGL renders are used and therefore this is an IO-bound task. 
\end{document}